%% file: main.tex
\documentclass[runningheads]{llncs}
\usepackage[T1]{fontenc}
\usepackage{graphicx}
\usepackage{algpseudocode}
\usepackage{algorithm}
\usepackage{xcolor}
\usepackage{amsmath} 
\usepackage{amssymb}  
\usepackage{wrapfig}

\begin{document}
\title{Interaction-aware Conformal Prediction \\ for Crowd Navigation}
%
%
\author{Zhe Huang\inst{1}\orcidID{0000-0002-2933-5379} \and
Tianchen Ji\inst{1}\orcidID{0000-0001-6547-2787} \and
Heling Zhang\inst{1}\orcidID{0009-0003-4891-7934} \and Fatemeh Cheraghi Pouria\inst{1}\orcidID{0009-0008-9708-7574} \and Katherine Driggs-Campbell\inst{1}\orcidID{0000-0003-3760-9859}  \and Roy Dong\inst{1}\orcidID{0000-0001-8034-4329}}
\authorrunning{Z. Huang et al.}
%
\institute{University of Illinois Urbana-Champaign, Urbana IL 61801, USA}

\maketitle              
\input{secs/0_abstract}
\renewcommand{\arraystretch}{1.2} 

\input{secs/1_intro}
\input{secs/2_related}
\input{secs/3_preliminaries}
\input{secs/4_method}
\input{secs/5_experiments}
\input{secs/6_conclusions}

\bibliographystyle{splncs04}
\bibliography{bib}
\end{document}

%% file: secs/0_abstract.tex
\begin{abstract}
During crowd navigation, robot motion plan needs to consider human motion uncertainty, and the human motion uncertainty is dependent on the robot motion plan. We introduce Interaction-aware Conformal Prediction (ICP) to alternate uncertainty-aware robot motion planning and decision-dependent human motion uncertainty quantification. ICP is composed of a trajectory predictor to predict human trajectories, a model predictive controller to plan robot motion with confidence interval radii added for probabilistic safety, a human simulator to collect human trajectory calibration dataset conditioned on the planned robot motion, and a conformal prediction module to quantify trajectory prediction error on the decision-dependent calibration dataset. Crowd navigation simulation experiments show that ICP strikes a good balance of performance among navigation efficiency, social awareness, and uncertainty quantification compared to previous works. ICP generalizes well to navigation tasks under various crowd densities. The fast runtime and efficient memory usage make ICP practical for real-world applications. Code is available at \url{https://github.com/tedhuang96/icp}.
\keywords{Human-Robot Interaction  \and Collision Avoidance.}
\end{abstract}

%% file: secs/1_intro.tex
\section{Introduction}
Despite decades of development in robot motion planning algorithms, it is only recently that mobile robots have started navigating through crowds and serving in our daily lives because of the advancement on data-driven modeling of human motion~\cite{hart1968formal},\cite{karaman2011sampling}, \cite{bemporad2007robust},\cite{alahi2016social}, \cite{huang2022learning}. While these human models are getting more accurate, how to effectively use predicted human motion for robot motion planning remains an open research problem. A classical paradigm performs motion planning by treating the predictions as if they are ground truth future human positions~\cite{huang2023neural}, \cite{liu2023intention}. However, there always exists prediction error, so the planned robot motion does not have any safety guarantees in this paradigm. Recent efforts are focused on calibration of the prediction error with uncertainty quantification techniques like conformal prediction~\cite{vovk2005cp}. As a calibration trajectory dataset is required for conformal prediction, previous works usually suffer from distribution shift on human motion due to (1) offline human-only simulation data collection which overlooks the difference of human reactions to robot from to other humans~\cite{lindemann2023safe}, or (2) impractical amount of online human-robot interaction data required for achieving asymptotic safety guarantee~\cite{dixit2023adaptive}.

We introduce Interaction-aware Conformal Prediction (ICP) to address the distribution shift issue by alternation between (1) robot motion planning based on the human motion uncertainty and (2) human motion uncertainty quantification by online human simulation conditioned on the robot motion plan.

In the initial step, ICP assumes predictions are ground truth and generates a robot motion plan (Fig.~\ref{fig-intro} b1, b2). Given the initial robot motion plan, ICP then starts iteration: (1) simulate multiple episodes of crowd motion by assuming the robot will execute the current plan to collect the calibration dataset dependent on the current plan (Fig.~\ref{fig-intro} c1); (2) perform conformal prediction to acquire the decision-dependent confidence interval radii (Fig.~\ref{fig-intro} c2); (3) plan the robot motion by using the current confidence interval radii as the decision-dependent probabilistic safety margin (Fig.~\ref{fig-intro} c3).

By explicitly capturing the mutual influence between the robot plan and the human motion uncertainty, ICP achieves a good tradeoff among navigation efficiency, social awareness, and uncertainty quantification in contrast to previous works in crowd navigation simulation experiments. We demonstrate that ICP generalizes well to crowd scenarios of different number of humans, and its fast runtime and small GPU memory usage show the readiness of real world applications.

\begin{figure*}[hbt!]
\centering
\includegraphics[width=0.8\textwidth]{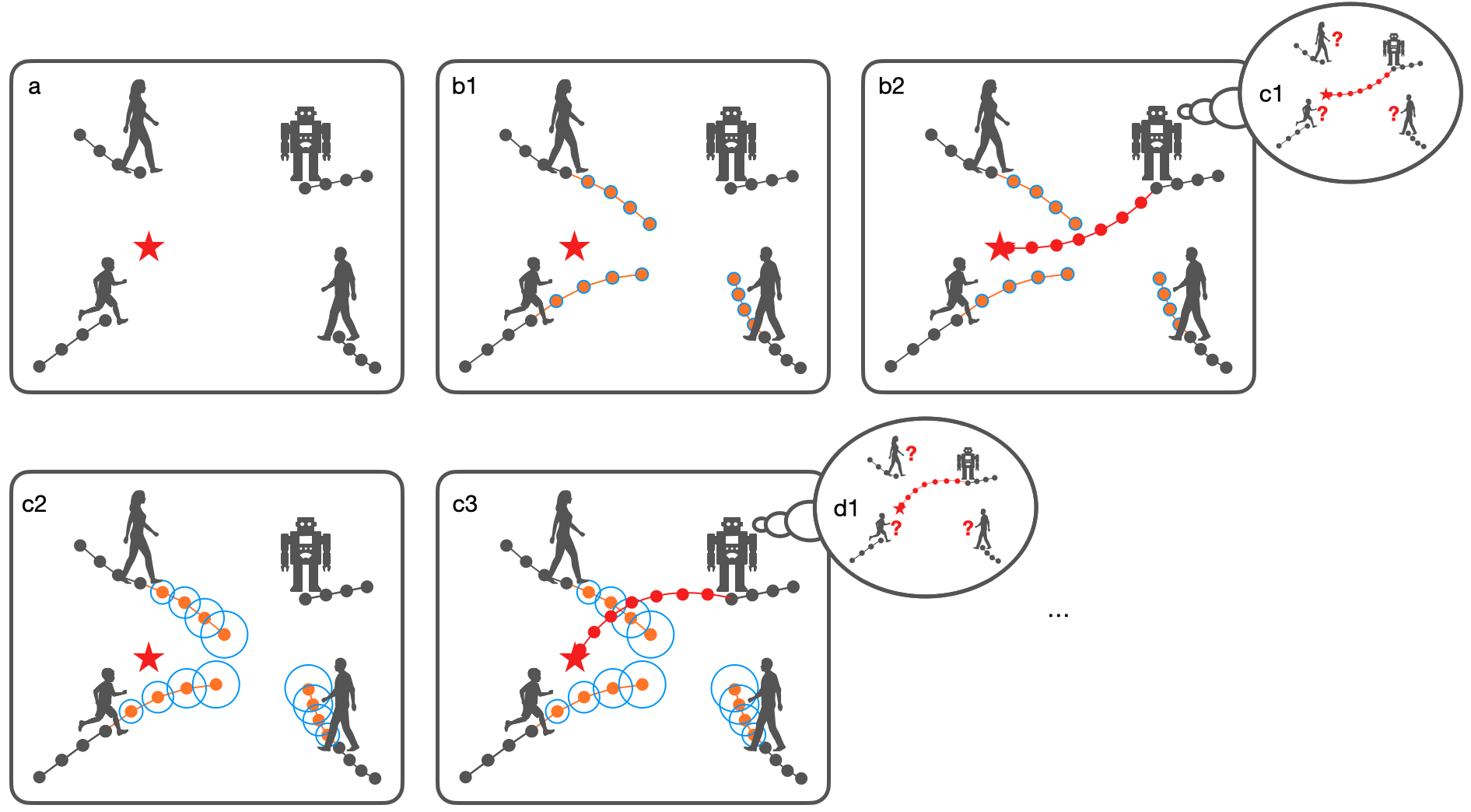}
\caption{Interaction-aware Conformal Prediction (ICP) iteratively quantifies uncertainty of human trajectory prediction by human motion simulation under the assumption that the robot would execute the latest plan, and plans robot motion with the conformal interval radii calibrated from the latest simulation dataset.} \label{fig-intro}
\end{figure*}

%% file: secs/2_related.tex
\section{Related Work}\label{sec-related}
\subsection{Conformal Prediction}
Conformal prediction is a statistical tool designed to produce reliable and 
valid prediction intervals or sets in machine learning. 
First introduced in \cite{vovk2005cp}, it offers 
a rigorous framework to quantify the uncertainty 
of predictions without making assumptions about 
the underlying distribution or predictive model \cite{angelopoulos2022cp}.
Due to its model-agnostic nature, conformal 
prediction has gained increasing popularity in various communities ranging from healthcare \cite{vazquez2022conformal}, \cite{olsson2022estimating},
to finance \cite{pmlr-v128-wisniewski20a}, \cite{KATH2021777}.

The application of conformal prediction has
also found great success in robotics, including combination with reachability analysis~\cite{dietterich2022conformal},~\cite{Muthali2023Multi-agent}, adding safety guarantees to trajectory prediction~\cite{lindemann2023safe},~\cite{dixit2023adaptive},~\cite{Sun2023ConformalPF}, and integration with robust control to find control laws with probabilistic safety guarantees~\cite{zhang2024distributionfree}. Most relevant to our work is the study by 
\cite{lindemann2023safe},
which applies conformal prediction and model predictive control (MPC) to plan robot motion with safety guarantees. Note that they run simulation of only human agents to collect a synthetic trajectory calibration dataset and perform conformal prediction offline. Thus, they compute fixed conformal interval radius as safety clearance for MPC. However, human agents adjust their behavior according to the robot action during human-robot interaction. A new robot plan will alter the distribution of human motion and break the guarantees offered by these fixed conformal sets.

Adaptive Conformal Prediction (ACP) attempts to address this issue by collecting human and robot trajectories, updating calibration datasets and adjusting failure probability on the fly~\cite{dixit2023adaptive}. A practical limitation of ACP is its asymptotic safety guarantee, where the average 
safety rate over all time steps approaches 
the designated safety rate as time goes to infinity. This indicates that a long warm-up period of online human-robot interaction data collection is necessary for achieving the asymptotic safety guarantee, which does not meet the efficiency requirement of crowd navigation applications. In contrast, our ICP algorithm offers distribution-free safety guarantees with robot plan refinement and human motion conformal set re-computation by leveraging online human simulation conditioned on robot plans.

\subsection{Crowd Navigation}
Various methods have been developed to enhance robot navigation in crowded environments. Reaction-based methods such as Optimal Reciprocal Collision Avoidance (ORCA) \cite{van2011reciprocal} treat agents as velocity obstacles, whereas methods like Social Force \cite{helbing1995social}, DS-RNN \cite{liu2021decentralized}, and \cite{liu2023intention} leverage attractive and repulsive forces or interaction-based graphs to model interactions between agents.

While these works have made notable contributions, their frameworks suffer from undetermined uncertainty quantification and are prone to safety problems. Hence, \cite{lindemann2023safe} and \cite{Muthali2023Multi-agent} have made use of conformal prediction to endow crowd navigation with probabilistic safety guarantees, where conformal prediction empowers their frameworks to deal with unknown data distribution.

\cite{Muthali2023Multi-agent} uses a specific prediction model and takes advantage of quantile regression models to generate approximate confidence intervals on predicted actions. Their approach is followed by implementing RollingRC, a conformal prediction method, to adjust composed intervals. Owing to the desirability of having confidence sets in the spatial domain, \cite{Muthali2023Multi-agent} and \cite{lin2024verification} use HJ reachability method to form reachable tubes for each agent. \cite{Muthali2023Multi-agent} obtains optimal trajectory for the ego agent by treating each agent's final forward reachable tube as a time-growing obstacle and maximizing the Hamiltonian.
Unlike previous works, ICP captures the interactions by iteratively computing conformal prediction sets and considering the effect of planner outputs on agents trajectories. 

\subsection{Model Predictive Control}
Model predictive control (MPC) is a control technique based on the iterative solution of an optimization problem~\cite{bemporad2007robust}. By using the system model and the current state, MPC plans the optimal control sequence based on a cost function. Due to its ability to handle multi-variable systems and state/input constraints, MPC has received considerable attention and has been studied within diverse research areas and application domains~\cite{garcia1989model,qin2003survey}. In robotics, MPC has been used to plan motions for mobile robots~\cite{chen2021interactive,sivakumar2021learned}, manipulators~\cite{incremona2017mpc}, and drones~\cite{ji2022robust,kamel2017robust}.

Existing MPC-based methods for robot navigation in social environments are often composed of two steps, prediction and planning, where the future trajectories of the surrounding agents are first predicted and then the robot action is planned by solving an optimization problem~\cite{chen2021interactive}. Park \textit{et al.} propose the model predictive equilibrium point control (MPEPC) for wheelchair robot navigation, where the uncertainty of obstacle motions is predefined and fixed~\cite{park2012robot}. Kamel \textit{et al.} employ a model-based controller to navigate a micro aerial vehicle (MAV) while avoiding collisions with other MAVs, where a constant velocity model is used for obstacle trajectory prediction and the obstacles are inflated for safety based on the uncertainty estimated by an extended Kalman filter~\cite{kamel2017robust}. A similar idea of enlarging the safety distance between the robot and a human based on the covariance of the estimated state is adopted by Toit \textit{et al.}, where several predefined dynamics are also explored for human trajectory prediction~\cite{du2011robot}. Chen \textit{et al.} propose an intention-enhanced ORCA (iORCA) as an advanced pedestrian motion model, which can dynamically adjust the preferred velocity of a pedestrian~\cite{chen2021interactive}. The predicted human trajectories from iORCA are then incorporated into an MPC framework to realize safe navigation in dense crowds. However, these approaches fail to take into account the effects of robot actions on future human trajectories, and thus the distribution shift on human behaviors exists, which can potentially lead to safety violations during execution.

%% file: secs/3_preliminaries.tex
\section{Introduction to Conformal Prediction}\label{sec-intro-cp}

\def\ddefloop#1{\ifx\ddefloop#1\else\ddef{#1}\expandafter\ddefloop\fi}
\def\ddef#1{\expandafter\def\csname bb#1\endcsname{\ensuremath{\mathbb{#1}}}}
\ddefloop ABCDEFGHIJKLMNOPQRSTUVWXYZ\ddefloop
\def\ddef#1{\expandafter\def\csname bf#1\endcsname{\ensuremath{\mathbf{#1}}}}
\ddefloop ABCDEFGHIJKLMNOPQRSTUVWXYZabcdefghijklmnopqrstuvwxyz\ddefloop
\def\ddef#1{\expandafter\def\csname bs#1\endcsname{\ensuremath{\boldsymbol{#1}}}}
\ddefloop ABCDEFGHIJKLMNOPQRSTUVWXYZabcdefghijklmnopqrstuvwxyz\ddefloop
\def\ddef#1{\expandafter\def\csname sf#1\endcsname{\ensuremath{\mathsf{#1}}}}
\ddefloop ABCDEFGHIJKLMNOPQRSTUVWXYZ\ddefloop
\def\ddef#1{\expandafter\def\csname c#1\endcsname{\ensuremath{\mathcal{#1}}}}
\ddefloop ABCDEFGHIJKLMNOPQRSTUVWXYZ\ddefloop

In this section, we provide a brief introduction of
conformal prediction. 
Consider a classic supervised learning setting with 
$n$ independent and identically distributed (i.i.d.) samples 
$(x_1, y_1), \dots, (x_n, y_n) \in \cX \times \cY$.
Let $\mu: \cX \to \cY$ be our predictor.
Given a new test sample $(x, y)$ drawn from the same 
distribution, we want to give a valid prediction region
for $y$ based on $x$.
Formally, for a given failure probability $\alpha \in (0, 1)$,
our goal is to find a set $C$ for $y$ such that
\begin{equation}
    Pr \left(
        y \in C
    \right)
    \ge
    1 - \alpha.
\end{equation}
Conformal prediction offers a simple way to find such $C$.
At its core, it uses the following simple fact 
about exchangeable random variables.
\begin{lemma}
    \label{lem:quantile}
    Let $X, X_1, \dots, X_n$ be exchangeable random variables. 
    Let $X_{(k)}$ be the $k$-th smallest value among
    $X_1, \dots, X_n$.
    Then we have
    \[
        Pr \left(
            X \le X_{(k)}
        \right)
        =
        \frac{k}{n+1}.
    \]
\end{lemma}
\begin{proof}
For simplicity, assume that there are no ties among 
$X, X_1, \dots, X_n$ almost surely. The same arguments would still
apply but with more complicated notations.

Let $f$ be the joint density of $X, X_1, \dots, X_n$. Consider the 
event $E$ that $\{X, X_1, \dots, X_n\} = \{x_0, x_1, \dots, x_n\}$.
By exchangeability of $X, X_1, \dots, X_n$, we have
\[
    f(x_0, x_1, \dots, x_n) 
    = 
    f(x_{\sigma(0)}, x_{\sigma(1)}, \dots, x_{\sigma(n)}).
\]
Thus, for any permutation $\sigma$ of $0, \dots, n$.
Thus, for any $i \in \{0, \dots, n\}$, we have
\begin{equation}
\begin{aligned}
    Pr \left(
        X = x_i | E
    \right)
    &=
    \frac{
        \sum_{\sigma: \sigma(0) = i}
        f(x_{\sigma(0)}, x_{\sigma(1)}, \dots, x_{\sigma(n)})
    }{
        \sum_\sigma
        f(x_{\sigma(0)}, x_{\sigma(1)}, \dots, x_{\sigma(n)})
    } \\
    &=
    \frac{n!}{(n+1)!} \\
    &=
    \frac{1}{n+1}.
\end{aligned}
\end{equation}
It then follows that
\[
    Pr \left(
        \left. X \le x_{(k)} \right| E
    \right)
    =
    Pr \left(
        \left. X \le X_{(k)} \right| E
    \right)
    = \frac{k}{n+1}.
\]
But this holds for any other event $E'$ such that 
$\{X, X_1, \dots, X_n\} = \{x'_0, x'_1, \dots, x'_n\}$.
Therefore, we can marginalize and get
\[
    Pr \left(
        X \le X_{(k)}
    \right)
    = \frac{k}{n+1}.
\]\qed
\end{proof}
Let $\ell: \cY \times \cY \to \bbR$ be the nonconformity measure
that quantifies the quality of our prediction.
For example, one commonly used nonconformity measure is
the Euclidean distance: 
$\ell(y, \mu(x)) = \|y - \mu(x)\|_2$.
Let $s_i$ denote the nonconformity score of the $i$-th sample,
i.e. $s_i := \ell(y_i, \mu(x_i))$,
and let $s$ be the nonconformity score of $(x, y)$.
The main result offered by conformal prediction is the following:
\begin{theorem}
    \label{thm:cp}
    Given i.i.d. samples $(x_1, y_1), \dots, (x_n, y_n)$,
    a test sample $(x, y)$ from the same distribution,
    a predictor $\mu$,
    a nonconformity measure $\ell$ and corresponding
    nonconformity scores $s_1, \dots, s_n$.
    For a given failure probability $\alpha \in (0, 1)$,
    let $q := \lceil (n + 1) (1 - \alpha) \rceil$.
    Then the set
    \[
        C = 
        \left\{
            \hat{y} \in \cY: \ell(\hat{y}, \mu(x)) 
            \le 
            s_{(q)}
        \right\}
    \]
    satisfies 
    \[
        Pr \left(
            y \in C
        \right)
        \ge 1 - \alpha.
    \]
\end{theorem}
\begin{proof}
    Since $s_1, \dots, s_n$ are i.i.d., they are 
    exchangeable.
    Then applying Lemma~\ref{lem:quantile}, we get
    \[
        Pr \left(
            y \in C
        \right)
        =
        Pr \left(
            s \le s_{(q)}
        \right)
        =
        \frac{q}{n+1}
        \ge
        1 - \alpha.
    \]\qed
\end{proof}
This result offers a potential way to construct safety sets with rigorous probability guarantees purely from samples. This alleviates the need to know the underlying distribution, which can be really complex for robotic systems.

%% file: secs/4_method.tex
\section{Method}\label{sec-method}

\subsection{Problem Formulation}

In crowd navigation, a mobile robot navigates to a goal position $g$ without colliding with any of $N$ humans moving in a shared 2D space. The position of the robot is denoted as $x_r^t$, and the positions of humans are denoted as $x_{h,i}^t$, $i\!\in\!\{1, \ldots, N\}$. The robot has a speed limit $v_{max}$. We need to plan velocity action $v_r^t$ given $\left(g, x_r^{1:t}, x_{h,1:N}^{1:t}\right)$, which the robot takes to reach the position at the next time step $x_r^{t+1}$.

\subsection{Preliminaries}

\textbf{Trajectory Prediction.} A trajectory prediction model $TP$ offers explicit modeling of human motion in a near future for planning socially aware robot motion. $TP$ takes robot and human positions in an observation window of length $T_{obs}$ as input, and predict human positions in a future time window of length $T_{pred}$. Our algorithm can take arbitrary models as $TP$. In this work, we use a learning-based trajectory prediction model Gumbel Social Transformer (GST)~\cite{huang2022learning}, which captures social interaction among multiple agents. Note we use a pre-trained GST with frozen weights which is only for inference.

\begin{equation}
    \hat{x}_{h,1:N}^{t+1:t+T_{pred}} = TP\left(x_r^{t-T_{obs}+1:t}, x_{h,1:N}^{t-T_{obs}+1:t}\right)
\end{equation}

\textbf{Human Simulator.} We use a human simulator to generate synthetic human trajectories for conformal prediction. We initialize the robot and the human positions in simulation as $x_r^{t}$ and $x_{h,1:N}^{t}$. Assuming that we have generated a robot plan, we enforce the robot to move along the planned robot trajectory $\bar{x}_{r}^{t+1:t+T_{plan}}$ and run ORCA~\cite{van2011reciprocal} to generate actions for each human to interact with the robot and human neighbors, so we can collect simulated human trajectories $x_{h, 1:N, sim}^{t+1:t+T_{plan}}$ at the end of one simulation episode. ORCA is a multi-agent motion planning algorithm which applies low-dimensional linear programming to compute velocities for all agents to reach respective goals without collision. ORCA has been extensively applied as a crowd simulator for training reinforcement learning crowd navigation policies that are successfully transferred to real world application without fine-tuning~\cite{chen2019crowd}, ~\cite{liu2021decentralized}, ~\cite{liu2023intention}. We randomize human goals and run multiple simulation episodes, and use a sliding window of length $T_{obs}+T_{pred}$ to split the collected trajectories into batches to create a trajectory prediction calibration dataset $\mathcal{D}$.

\textbf{Conformal Prediction.} The trajectory prediction calibration dataset $\mathcal{D}$ includes $M$ samples of past robot trajectories, past human trajectories and future human trajectories. We use $TP$ to make trajectory prediction for each sample, and compute prediction errors of each human's position at each prediction time step in each sample.

\begin{equation}
    e_{i}^{\tau, j} = ||\hat{x}_{h,i}^{T_{obs}+\tau, j} - x_{h,i}^{T_{obs}+\tau, j}||_2, i\!\in\!\{1, \ldots, N\}, j\!\in\!\{1, \ldots, M\}, \tau\!\in\!\{1, \ldots, T_{pred}\}
\end{equation}

We aggregate the prediction errors in terms of the prediction time step $\mathcal{E}^\tau = \{e_{i}^{\tau, j}\}_{i=1:N,j=1:M}$. We assume the errors in $\mathcal{E}^\tau$ are from an exchangeable probability distribution, and sort the errors in a non-decreasing order $\{{e}_{(l)}^{\tau}\}_{l=1:N\times M}$. For a given failure probability $\alpha$, we define confidence interval radius at each prediction time step

\begin{equation}
r_{cp}^{\tau} := {e}_{(\lceil(1-\alpha)(N\times M)\rceil)}, \quad \tau\!\in\!\{1, \ldots, T_{pred}\}
\end{equation}
By treating the trajectory prediction error for any human $i$ at the current time step $t$ as the $N\times M+1$th sample from the exchangeable error distribution, we achieve probabilistic guarantees for prediction at each time step
\begin{equation}
Pr\left(||\hat{x}_{h,i}^{t+\tau}-x_{h,i}^{t+\tau}||_2 \leq
r_{cp}^{\tau}\right) = Pr\left({e}_{(N\times M+1)} \leq
{e}_{(\lceil(1-\alpha)(N\times M)\rceil)}\right) \geq 1 - \alpha
\end{equation}

Note that in ICP, the calibration dataset is composed of the robot plan generated by model predictive control, and the human trajectories simulated based on the assumption that the robot will execute the generated plan.

\textbf{Model Predictive Control.} The model predictive control module (MPC) plans robot motion to reach the goal while satisfying dynamics constraints, control limit constraints, and collision avoidance constraints. In the collision avoidance constraints, $r_r$ is robot radius and $r_h$ is human radius. The conformal interval radii $r_{cp}^{\tau}$'s are incorporated in the collision avoidance constraints to inform MPC about the uncertainty on the predicted human positions. In the optimization problem presented in Equation~\ref{mpc-optimization-v2}, the step-wise cost function includes the goal-reaching cost, the velocity jerk cost, and a regularization cost from the last round of MPC solution, since our algorithm iteratively runs MPC. The regularization cost helps constrain the change of robot plan through iterations, which prevents the drastic oscillation of conformal interval radii and facilitates convergence. The regularization cost is ignored when it is the first round of MPC. Note the collision avoidance constraints make the optimization problem non-convex, but there are usually feasible solutions in practice. 

\vspace{-10pt}
\begin{equation}
\label{mpc-optimization-v2}
\begin{aligned}
\operatorname*{minimize}_{\bar{\mathbf{x}}_r, \bar{\mathbf{v}}_r} \,
& \sum_{\tau=t}^{t+T_{mpc}} \omega_\text{g} ||\bar{x}_r^{\tau} - g||_2^2
+
\sum_{\tau=t}^{t+T_{mpc}-2} \omega_\text{v} ||\bar{v}_r^{\tau + 1} - \bar{v}_r^{\tau}||_2^2 \\
& \qquad \qquad \qquad \qquad \qquad \qquad \qquad \qquad +
\sum_{\tau=t}^{t+T_{mpc}} \omega_\text{reg} ||\bar{x}_r^{\tau} - \bar{x}_{r, k-1}^{\tau}||_2^2 \\
\text{subject to} \quad &\bar{x}_r^{\tau+1} = \bar{x}_r^\tau + \bar{v}_r^\tau \, \Delta T, \quad \tau = t, \ldots, t+T_{mpc}-1, \\
&||\bar{v}_r^\tau||_2 \leq v_{max}, \quad \tau = t, \ldots, t+T_{mpc}-1, \\
&||\bar{x}_r^{t + \tau} - \hat{x}_{h, i}^{t + \tau}||_2 \geq r_r + r_h + r_{cp}^\tau, \quad i = 1, \dots, N, \; \; \tau = 1, \dots, T_{pred}, \\
& \bar{x}_r^t = x_r^t.
\end{aligned}
\end{equation}
where the optimization variable $\bar{\mathbf{x}}_r = (\bar{x}_r^t, \bar{x}_r^{t + 1}, \dots \bar{x}_r^{t + T_{mpc}})$ is the planned robot trajectory, $\bar{\mathbf{v}}_r = (\bar{v}_r^t, \bar{v}_r^{t + 1}, \dots, \bar{v}_r^{t + T_{mpc} - 1})$ is the planned robot velocity, and $\Delta T$ is the time interval between two points on the planned robot trajectory.

\subsection{Interaction-aware Conformal Prediction}
To explicitly address the mutual influence between the robot and the humans during the interaction, Interaction-aware Conformal Prediction (ICP) alternates model predictive control for robot motion planning and conformal prediction for human trajectory prediction, which is presented in Algorithm~\ref{algo:ICP}.

At time $t$, we first feed observed robot and human trajectories into $TP$ to generate predictions of human trajectories $\hat{x}_{h, 1:N}^{t+1:t+T_{pred}}$. With the predictions, we run MPC by assuming confidence interval radii as zero and obtain a nominal robot trajectory $\bar{x}_{r, 0}^{t+1:t+T_{mpc}}$. This nominal robot plan does not have any safety guarantees yet, because no uncertainty quantification has been done for the predicted human trajectories used in MPC.

We introduce an inner iteration to iteratively calibrate the uncertainty of the human trajectory prediction and finetune the robot plan. We simulate crowd motions reacting to the most recent robot plan $\bar{x}_{r, {k-1}}^{t+1:t+T_{mpc}}$ and to collect a trajectory prediction calibration dataset $\mathcal{D}_k$. We perform conformal prediction for the prediction model $TP$ on the calibration dataset $\mathcal{D}_k$ to calculate the conformal interval radii $r_{cp, k}^{1:T_{pred}}$. We then run MPC with the collision avoidance constraints incorporating the updated conformal interval radii, and the regularization cost with respect to the latest MPC solution to generate a new robot plan $\bar{x}_{r, {k}}^{t+1:t+T_{mpc}}$. The new robot plan is used to initiate the next iteration until an iteration limit is reached, or the robot plan and the conformal prediction radii converge. The robot will execute the actions of next $T_{exec}$ steps $\bar{v}_{r, K}^{t:t+T_{exec}-1}$ of the final robot plan generated from the iterations.
\vspace{-15pt}

\begin{algorithm}[hbt!]
    \caption{Interaction-aware Conformal Prediction}\label{algo:ICP}
    \begin{algorithmic}
    \State Load a pre-trained \textbf{trajectory prediction} model $\textit{\textbf{TP}}$.
    \State Set terminal constraints $g$, and control limit constraints $v_{max}$ for model predictive control.
    \For {$t = 1$ to $T$}
        \State Predict future human trajectories $\hat{x}_{h, 1:N}^{t+1:t+T_{pred}}$ by taking past robot and human trajectories as input to $\textit{\textbf{TP}}$.
        \State Initialize conformal interval radii $r_{cp, 0}^{1:T_{pred}}$ as zero.
        \State Set the initial constraints of model predictive control with robot position $x_r^t$.
        \State Set the collision avoidance constraints of model predictive control with predicted human trajectories $\hat{x}_{h, 1:N}^{t+1:t+T_{pred}}$, human radius, robot radius, and the initialized conformal interval radii $r_{cp, 0}^{1:T_{pred}}$.
        \State Run \textbf{model predictive control} to generate action sequence $\bar{v}_{r, 0}^{t:t+T_{mpc}-1}$ and corresponding robot trajectory $\bar{x}_{r, 0}^{t+1:t+T_{mpc}}$.
        \For {$k = 1$ to $K$}
            \State \textbf{Simulate human motion} by assuming robot executes the plan $\bar{x}_{r, k-1}^{t+1:t+T_{mpc}}$ with multiple runs, and collect a trajectory prediction calibration dataset $\mathcal{D}_k$.
            \State Run \textbf{conformal prediction} by evaluating $TP$ on $\mathcal{D}_k$, collecting trajectory prediction errors, and computing the $k$th conformal interval radii $r_{cp, k}^{1:T_{pred}}$ with safe probability $1-\alpha$.
            \State Update the collision avoidance constraints of model predictive control with the $k$th conformal interval radii $r_{cp, k}^{1:T_{pred}}$.
            \State Run \textbf{model predictive control} to generate the $k$th action sequence $\bar{v}_{r, k}^{t:t+T_{mpc}-1}$ and the $k$th robot trajectory $\bar{x}_{r, k}^{t+1:t+T_{mpc}}$.
        \EndFor
        \State The robot executes the actions $\bar{v}_{r, K}^{t:t+T_{exec}-1}$.
    \EndFor
 \end{algorithmic}
 \end{algorithm}
\vspace{-15pt}

When Algorithm~\ref{algo:ICP} converges in the sense that the planned trajectory
from the last iteration induces the same human behavior as 
the previous iteration, then we have the following safety guarantee
\begin{theorem}
    Assume that Algorithm~\ref{algo:ICP} converges at time $t$,
    and the optimization problem in Equation \ref{mpc-optimization-v2} is feasible
    at $t$ with prediction horizon $T_{pred}$. Then the planned trajectory $\bar{x}_r^{t+1:t+T_{mpc}}$ satisfies
    \[
        Pr \left(
            ||\bar{x}_r^{t+\tau} - x_{h, i}^{t+\tau}||_2
            \geq r_r + r_h,
            \forall
            \tau \in
            [T_{pred}],
            \forall
            h \in
            [N]
        \right)
        \ge
        1 - \alpha N T_{pred}.
    \]
\end{theorem}
\begin{proof}
    By the convergence assumption and Theorem~\ref{thm:cp}, 
    for each $h \in \{1, \dots, N\}$ and each 
    $\tau \in \{1, \dots, T_{pred}\}$,
    we have
    \[
        Pr \left(
            ||\hat{x}_{h, i}^{t+\tau} - x_{h, i}^{t+\tau}||_2
            \le r^{\tau}_{cp}
        \right)
        \ge
        1 - \alpha.
    \]
    Further, from the optimization constraints in Equation~\ref{mpc-optimization-v2}, we have
    \[
        ||\bar{x}_r^{t+\tau} - \hat{x}_{h, i}^{t+\tau}||_2
        \geq r_r + r_h + r^{\tau}_{cp},
    \]
    and thus
    \[
        Pr \left(
            ||\bar{x}_r^{t+\tau} - x_{h, i}^{t+\tau}||_2
            \geq r_r + r_h
        \right)
        \ge
        1 - \alpha.
    \]
    Then taking the union bound over $\tau$ and $h$, we get our result.
    \qed
\end{proof}

%% file: secs/5_experiments.tex
\section{Experiments}\label{sec-experiments}

\subsection{Experiment Setup}

\textbf{Simulation Environment.} We conduct crowd navigation simulation experiments for evaluation. In each test case, the initial and goal positions of the robot and the humans are randomized. The humans are controlled by ORCA~\cite{van2011reciprocal} to react to each other and the robot. Both robot radius $r_r$ and human radius $r_h$ are set as $0.4\,m$. Both robot max speed and human max speed are set as $v_{max} = 1 \,m/s$. We apply holonomic kinematics to both robot and humans. One time step $\Delta t$ is set as $0.25\,s$. We run 100 test cases for each configuration of ICP and each baseline to report the performance.

\textbf{Baselines.} We show effectiveness of alternating conformal prediction and planning for interaction awareness by comparing ICP to the following methods: 

\begin{itemize}
    \item Offline Conformal Prediction (OffCP) \cite{lindemann2023safe}: an offline method to pre-compute conformal interval radius, and use the fixed conformal interval radius throughout planning and execution. Note OffCP performs simulation of crowd motion among human agents without the robot agent offline, of which the data distribution ignores interaction between robot and humans.  
    \item Adaptive Conformal Prediction (ACP) \cite{dixit2023adaptive}: an online method which adaptively modifies the failure probability $\alpha$ to adjust the conformal interval radius based on conformal prediction from the dataset formed by most recent human and robot trajectories. The original work did not discuss how to compute gradient on the failure probability for multiple humans scenario. We propose two versions: ACP-A averages the gradients computed for each human whether their trajectory prediction error is within the conformal interval radius; ACP-W takes the worst possible gradient, by treating all humans have prediction error beyond the conformal interval radius whenever any of them has prediction error beyond the conformal interval radius. Note that ACP-W retains the asymptotically probabilistic safety guarantees, while ACP-A does not.
    \item Optimal Reciprocal Collision Avoidance (ORCA) \cite{van2011reciprocal}: a reactive navigation method based on the assumption that all agents are velocity obstacles which make similar reasoning on collision avoidance. The robot ORCA configuration is set the same as the human ORCA configuration in simulation experiments. 
\end{itemize}

\textbf{Metrics.} We use metrics in terms of navigation efficiency, social awareness and uncertainty quantification to comprehensively evaluate our method compared to the baselines. We use success rate (SR), robot navigation time (NT) and navigation path length (PL) as performance metrics. SR is the ratio of test cases where the robot successfully reaches the goal without collision with humans. We use intrusion time ratio (ITR) and social distance during intrusion (SD) adopted in \cite{liu2023intention} as social awareness metrics. ITR is the ratio between the number of time steps when the robot collide with any human's ground truth future positions in the prediction horizon window of $T_{pred}$, and the number of time steps for robot to reach the goal. SD is the average distance between the robot and the closest human during intrusion. We use coverage rate (CR) to check the performance on uncertainty quantification. For each human $i$ at time $t$, we check whether the trajectory prediction within the prediction horizon are within the computed conformal interval radii of the ground truth future positions. We average across all $N$ humans through the whole time period $T$ to obtain the coverage rate of one test case. Note the unit of NT is second, and the unit of PL and SD is meter.
\begin{equation}
    CR = \frac{1}{T \times N} \sum_{t=1}^{T} \sum_{i=1}^{N} \prod_{\tau=1}^{T_{pred}} \mathbf{1}[||\hat{x}_{h, i}^{t+\tau} - x_{h, i}^{t+\tau}||_2 < r_{cp}^{\tau}]
\end{equation}

\textbf{Implementation Details.} We set the number of humans as 5, 10, 15, and 20, and run 100 test cases for each crowd setup. We set the number of iterations as 1, 3, and 10 for ICP. Note 1 iteration is also interaction-aware, because it includes two rounds of MPC and one round of simulation. The size of the calibration dataset is adjusted by controlling the number of episodes to run in the simulator, which we define as calibration size (CS). The calibration size is tested across 2, 4, 8, 16, 32, and 64 with 100 test cases run for each configuration. We evaluate two types of execution scheme (ES). The first type is we execute a sequence of actions $\bar{v}_{r, K}^{t:t+T_{pred}-1}$ of length $T_{pred}$ from the robot plan, which is named as Pred-Step Execution (PSE). The second type is we execute only one step of action $\bar{v}_{r, K}^{t}$, which is named as Single-Step Execution (SSE). We set prediction horizon $T_{pred}$ as 5 time steps (1.25 second). The failure probability $\alpha$ is set as 0.05 for all conformal prediction related methods. Based on union bound argument, we can bound the probability that a human trajectory stays within the conformal radius of the entire prediction horizon as follows
\begin{equation}
    \label{eq:coverage}
    Pr\left(
        ||\hat{x}_{h,i}^{t+\tau}-x_{h,i}^{t+\tau}||_2 
        \leq
        r_{cp}^{\tau},
        \forall
        \tau \in \{1, \dots, T_{pred}\}
    \right) 
    \geq 1 - \alpha T_{pred}.
\end{equation}
Thus, the lower bound of coverage rate is $1-0.05\times 5 = 0.75$.

The human simulator used in ICP is run separately from the experiment simulator. In the ICP simulator, we add noises to human goals at random steps during each episode to add randomness to the human behavior and diversify the collected data for calibration. We run the ICP simulator in multiple threads to parallelize the calibration data collection process. When the calibration size is less than 8, the number of threads is equal to the calibration size. Otherwise, we set the number of threads as 8. 

We set the weight parameters in the cost function of MPC $\omega_{g}$ as 1, $\omega_{v}$ as 5, and $\omega_{reg}$ as 0.5 across ICP, ACP and OffCP. To handle the cases when the constraints are too extreme and there are no feasible MPC solutions, we cache the most recent feasible plans for execution. We tune parameters of ACP-A and ACP-W, where the learning rate of ACP-A is set as 0.05, and the learning rate of ACP-W is set as 0.01. The time window used for online calibration dataset collection is set as 30 time steps (7.5 second) for both ACP-A and ACP-W. 

\subsection{Experiment Results}
We report quantitative performance of ICP with different configurations and baselines in 10-human crowd test cases for both PSE and SSE scheme in Table~\ref{table-quantitative}. 
\begin{table}[hbt!]
\caption{Performance of ICP with different configurations and baseline algorithms in 100 crowd navigation test cases of 10 humans. The subscript of ICP denotes the index of configuration. NI denotes number of iterations, CS denotes calibration size, ES denotes execution scheme, SR denotes success rate, ITR denotes intrusion time ratio, SD denotes social distance, PL denotes robot path length, NT denotes robot navigation time. CR denotes coverage rate, where 0.75 is the lower bound corresponding to the failure probability $\alpha$ as 0.05. The best performance for PSE and SSE configurations are independently highlighted.}\label{table-quantitative}
\begin{center}
\fontsize{7}{10}\selectfont
\begin{tabular}{c|ccc|cccccc}
    \hline
    Method & NI & CS & ES & SR$\uparrow$ & ITR$\downarrow$ & SD$\uparrow$ & PL$\downarrow$ & NT$\downarrow$ & CR$\uparrow$ (0.75) \\
    \hline
    \, ORCA\, & \,-\, & \,-\, & \,-\, & \,\textbf{0.99}\, & \,0.26$\pm$0.17\, & \,1.23$\pm$0.10\, & \,12.66$\pm$1.24\, & \,17.48$\pm$4.02\, & \,- \\
    \, OffCP\, & \,-\, & \,8\, & \,PSE\, & \,\textbf{0.99}\, & \,0.17$\pm$0.14\, & \,1.28$\pm$0.14\, & \,\textbf{12.57$\pm$0.62}\, & \,11.35$\pm$1.65\, & \,0.85$\pm$0.08 \\
    \, ACP-A\, & \,-\, & \,-\, & \,PSE\, & \,\textbf{0.99}\, & \,0.16$\pm$0.13\, & \,1.29$\pm$0.14\, & \,12.89$\pm$1.38\, & \,11.60$\pm$2.29\, & \,0.89$\pm$0.04 \\
    \, ACP-W\, & \,-\, & \,-\, & \,PSE\, & \,0.98\, & \,0.16$\pm$0.14\, & \,1.30$\pm$0.14\, & \,12.96$\pm$1.43\, & \,11.66$\pm$2.32\, & \,0.91$\pm$0.04 \\
    \, ICP$_1$\, & \,3\, & \,8\, & \,PSE\, & \,0.98\, & \,0.15$\pm$0.12\, & \,\textbf{1.32$\pm$0.16}\, & \,12.59$\pm$0.67\, & \,11.12$\pm$1.25\, & \,\textbf{0.93$\pm$0.05} \\
    \hline
    \, ICP$_2$\, & \,3\, & \,2\, & \,PSE\, & \,0.97\, & \,0.15$\pm$0.12\, & \,\textbf{1.32$\pm$0.15}\, & \,12.68$\pm$0.98\, & \,11.21$\pm$1.64\, & \,\textbf{0.93$\pm$0.05} \\
    \, ICP$_3$\, & \,3\, & \,4\, & \,PSE\, & \,0.96\, & \,\textbf{0.14$\pm$0.12}\, & \,\textbf{1.32$\pm$0.16}\, & \,12.61$\pm$0.61\, & \,11.13$\pm$1.32\, & \,\textbf{0.93$\pm$0.05} \\
    \, ICP$_4$\, & \,3\, & \,16\, & \,PSE\, & \,0.97\, & \,0.15$\pm$0.12\, & \,\textbf{1.32$\pm$0.16}\, & \,12.59$\pm$0.62\, & \,\textbf{11.09$\pm$1.28}\, & \,\textbf{0.93$\pm$0.05} \\
    \, ICP$_5$\, & \,3\, & \,32\, & \,PSE\, & \,0.96\, & \,0.15$\pm$0.12\, & \,\textbf{1.32$\pm$0.16}\, & \,12.61$\pm$0.78\, & \,11.11$\pm$1.35\, & \,\textbf{0.93$\pm$0.05} \\
    \, ICP$_6$\, & \,3\, & \,64\, & \,PSE\, & \,0.96\, & \,0.15$\pm$0.12\, & \,\textbf{1.32$\pm$0.16}\, & \,12.62$\pm$0.78\, & \,11.12$\pm$1.36\, & \,\textbf{0.93$\pm$0.05} \\
    \hline
    \, ICP$_7$\, & \,1\, & \,8\, & \,PSE\, & \,0.98\, & \,0.16$\pm$0.12\, & \,\textbf{1.32$\pm$0.16}\, & \,12.61$\pm$0.61\, & \,11.16$\pm$1.28\, & \,\textbf{0.93$\pm$0.05} \\
    \, ICP$_8$\, & \,10\, & \,8\, & \,PSE\, & \,\textbf{0.99}\, & \,0.16$\pm$0.12\, & \,1.30$\pm$0.15\, & \,12.58$\pm$0.55\, & \,11.10$\pm$1.20\, & \,\textbf{0.93$\pm$0.05} \\
    \hline
    \hline
    \, OffCP\, & \,-\, & \,8\, & \,SSE\, & \,\textbf{1.00}\, & \,0.16$\pm$0.13\, & \,1.29$\pm$0.14\, & \,\textbf{12.39$\pm$0.56}\, & \,\textbf{11.25$\pm$1.22}\, & \,0.82$\pm$0.08 \\
    \, ACP-A\, & \,-\, & \,-\, & \,SSE\, & \,0.98\, & \,0.15$\pm$0.12\, & \,1.27$\pm$0.14\, & \,12.58$\pm$0.79\, & \,11.58$\pm$1.45\, & \,0.91$\pm$0.03 \\
    \, ACP-W\, & \,-\, & \,-\, & \,SSE\, & \,0.96\, & \,\textbf{0.14$\pm$0.11}\, & \,1.30$\pm$0.13\, & \,13.46$\pm$2.59\, & \,12.95$\pm$4.11\, & \,\textbf{0.96$\pm$0.02} \\
    \, ICP$_9$\, & \,3\, & \,8\, & \,SSE\, & \,0.97\, & \,\textbf{0.14$\pm$0.12}\, & \,\textbf{1.31$\pm$0.15}\, & \,12.58$\pm$0.76\, & \,11.33$\pm$1.40\, & \,0.90$\pm$0.04 \\
    \hline
\end{tabular}
\end{center}
\end{table}

In PSE scheme, we see the coverage rate of ICP is consistently higher than the baselines to provide better safety guarantees, while ICP still achieves state-of-the-art performance in navigation and social-awareness. Fig.~\ref{fig-cr-nt-mean-std} shows that ICP reaches a sweet spot of the lowest navigation time and the highest coverage rate in PSE scheme regardless of the crowd density of the scenes. We find that OffCP has a consistently lower coverage rate than online methods including ICP and ACP. This matches our claim that OffCP calibrates human motion uncertainty of based on samples from a shifted distribution which ignores the interaction between the robot and the humans, and would lead to inaccurate uncertainty quantification.

\begin{figure}[hbt!]
\includegraphics[width=\textwidth]{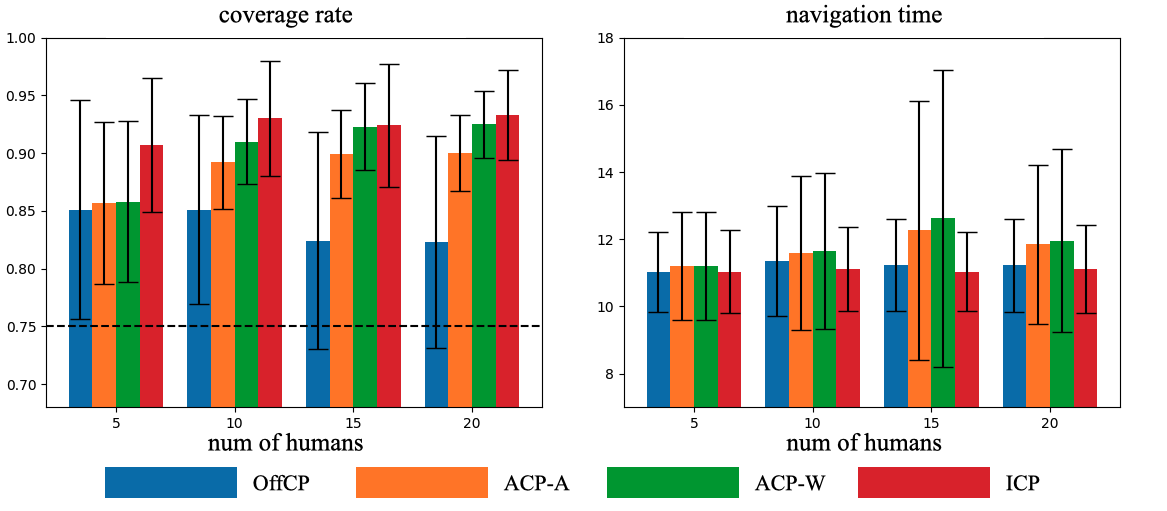}
\caption{Coverage rate (CR) and robot navigation time (NT) of algorithms with Pred-Step Execution scheme in crowd scenes of different number of humans. The error bars denote the standard deviation. The unit of robot navigation time is second. We use ICP$_1$ among all ICPs with PSE configurations for comparison.} \label{fig-cr-nt-mean-std}
\end{figure}

Fig.~\ref{fig-test-case} presents the comparison of performance between ICP$_1$ and ACP-W for each test case in both PSE and SSE schemes. We clearly see the effect of number of humans on the distribution of coverage rate over test cases in ACP-W in both left of Fig.~\ref{fig-cr-nt-mean-std} and the top row of Fig.~\ref{fig-test-case}, where there is a notable number of violations of coverage rate lower bound in 5-human test cases. This is due to the fact that ACP collects the calibration dataset on the fly, of which the size is insufficient and dependent on the number of humans. In contrast, the coverage rate of test cases run with ICP is both high and stable as the simulation provides abundant interaction-aware samples for uncertainty calibration even when the number of humans is low. We argue this is also the reason why ICP$_{2-6}$ whose calibration size spans from 2 to 64 have similar performance across all metrics. Running 2 simulation episodes turns out to be sufficient to calibrate human motion uncertainty when the robot needs to navigate through 10 humans.

\begin{figure}[hbt!]
\includegraphics[width=\textwidth]{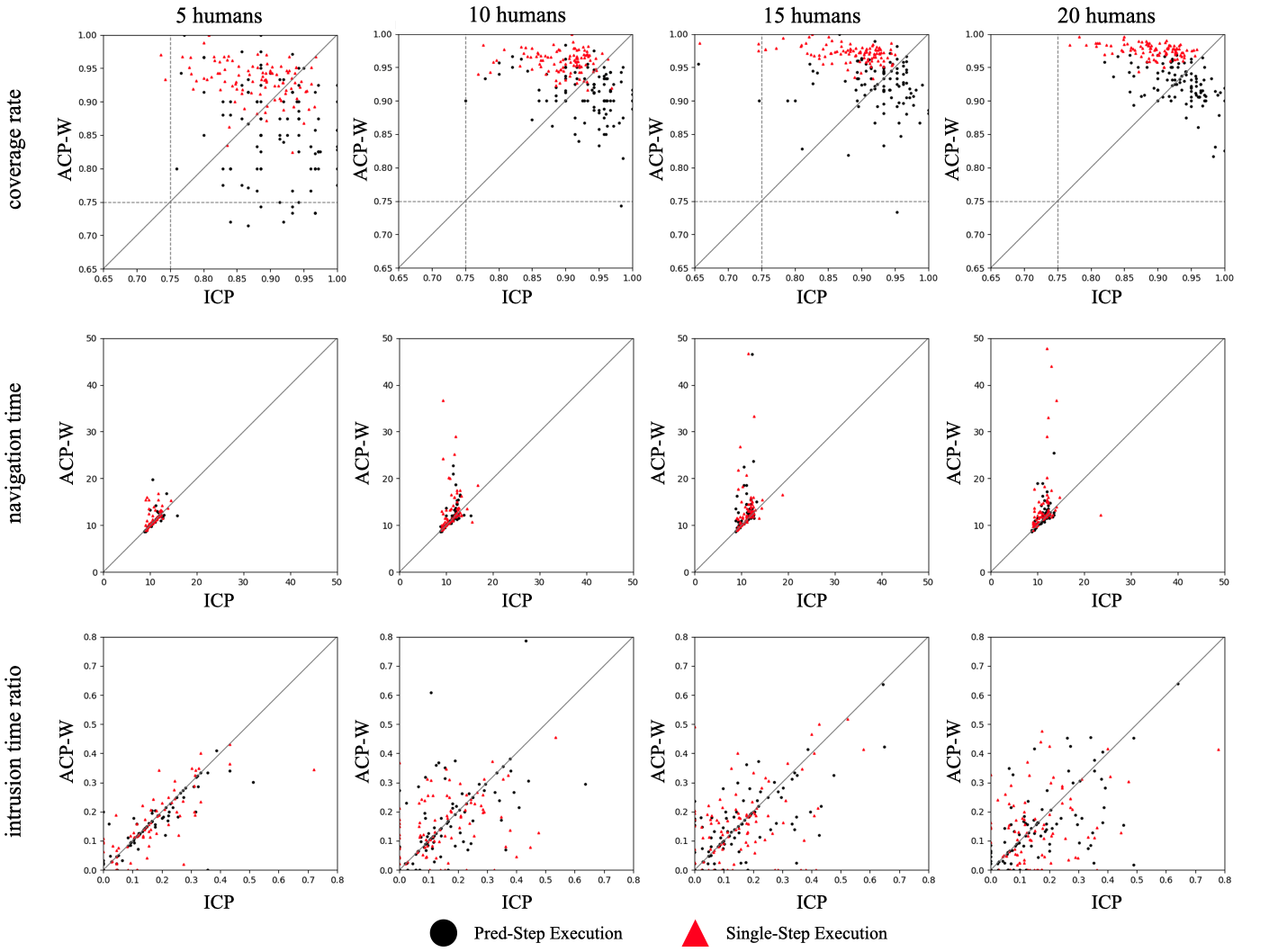}
\caption{Performance comparison between ICP and ACP-W for both Pred-Step Execution (PSE) Scheme and Single-Step Execution (SSE) Scheme. One black dot is for one test case in the Pred-Step Execution, where ICP and ACP-W share the same configurations on start and goal positions for the robot and the humans. One red triangle is for one test case in the Single-Step Execution. The X value of a black dot or a red triangle shows the performance of ICP, and the Y value shows the performance of ACP-W. Note we use ICP$_1$ for PSE comparison.} \label{fig-test-case}
\end{figure}

It is surprising that the top row of Fig.~\ref{fig-test-case} indicates the coverage rate of ACP-W is in the SSE scheme is better than in the PSE scheme, which is reverse to our expectation as SSE makes the lower bound of the coverage rate not hold anymore. Nevertheless, The coverage rate of ACP-W comes at the price of unstable navigation time performance in contrast to ICP, which is shown in the middle row of Fig.~\ref{fig-test-case}. Regarding the social metrics, the bottom row of Fig.~\ref{fig-test-case} shows that ACP-W and ICP tend to have more comparable intrusion time ratio per test case when the number of humans are lower (e.g., 5). We reason that that lower number of humans indicates simpler interaction patterns, which are less sensitive to different robot plans from ACP and ICP.

Fig.~\ref{fig-snapshot} demonstrates crowd navigation of OffCP, ACP-W, and ICP$_9$ in SSE scheme. We see that the conformal interval radius of ICP during the crowd-robot interaction ($t = 5$) is greater than before ($t = 2.5$) and after ($t= 7.5$), which illustrates that the human motion uncertainty is higher when the crowd-robot interaction is more involved. ACP-W exhibit similar trend by implicitly capturing the mutual influence with online calibration dataset collection. However, the higher coverage rate of ACP-W is at the price of excessive collision constraints caused by the large confidence interval radius, which leads to overly conservative and deviated robot motion. OffCP fails to capture the mutual influence and has fixed small confidence interval radius assuming no presence of the robot through the whole episode, and results in robot motion similar to treating the predictions as ground truth future positions, which exhibits less social awareness.

\begin{figure}[hbt!]
\includegraphics[width=\textwidth]{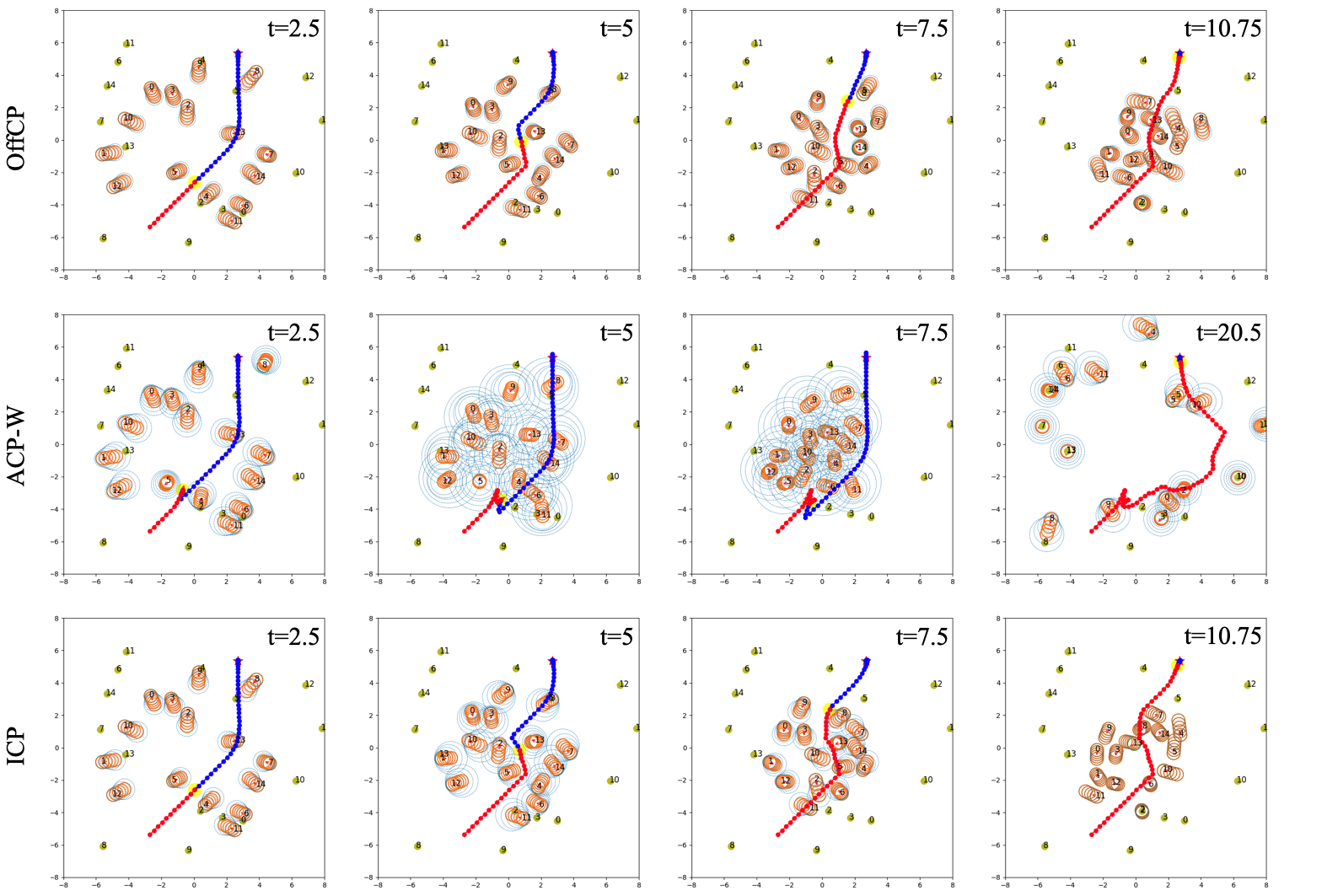}
\caption{Snapshots of one crowd navigation test case in SSE scheme. We use ICP$_9$ for ICP visualization. The last column shows the snapshots whe n the robot reaches the goal. The bright yellow disk denotes the robot. The star denotes the robot goal. The orange circles with indices denote the humans with the predicted positions. The bright blue circles denote human radius bloated by the confidence interval radius. The red dots denote the history of the robot positions. The blue dots denote the generated plan to be executed by the robot. The dark yellow dots with indices denote the corresponding human's goal.} \label{fig-snapshot}
\end{figure}

We investigate the practicality of ICP in real world applications by checking the runtime and GPU memory usage. We find that ICP with appropriate configurations can be readily applied in real-time in either PSE (1.25 sec) or SSE (0.25 sec) scheme. The GPU memory usage of the algorithm is also manageable for a standard commercial GPU (e.g., GeForce RTX 2080 with total memory  8192 MiB).

\section{Limitations}\label{sec-limitations}
As demonstrated in Fig.~\ref{fig-snapshot}, higher coverage rates may cause excessively constrained conditions, which leads to infeasible solutions. When this occurs, we use a cached plan which may not remain optimal. If the cached plan is used beyond $T_{pred}$ steps, the probabilistic safety guarantee no longer exists, and the robot and the humans are susceptible to collision. To address this challenge, we are interested in exploring the integration of the adaptive failure probability idea from ACP into ICP, where the adaptation is dependent on the feasibility of optimization problem in MPC.

The performance comparison between ICP$_1$, ICP$_7$, and ICP$_8$ in Table~\ref{table-quantitative} indicates that having 1 iteration of ICP can already capture interaction between robot and humans well. We argue that this is because ORCA is used both in the human simulator of ICP algorithm and for generating human motion in test scenarios. We expect that the sim-to-real gap between the human simulator and the real world human behavior pattern would require more iterations for better performance, which is left for future work.

%% file: secs/6_conclusions.tex
\section{Conclusions}\label{sec-conclusions}
We present Interaction-aware Conformal Prediction (ICP) to explicitly address the mutual influence between robot and humans in crowd navigation problems. We achieve interaction awareness by proposing an iterative process of robot motion planning based on human motion uncertainty and conformal prediction of the human motion dependent on the robot motion plan. Our crowd navigation simulation experiments show ICP strikes a good balance of performance among navigation efficiency, social awareness, and uncertainty quantification compared to previous works. ICP generalizes well to navigation tasks across different crowd densities, and its fast runtime and manageable memory usage indicates potential for real-world applications.

In future work, we will address infeasible robot planning solutions with adaptive failure probability and conduct real-world crowd navigation experiments to evaluate the effectiveness of ICP. As ICP is a task-agnostic algorithm, we would like to explore its applications in manipulation settings, such as collaborative manufacturing.

\begin{credits}
\subsubsection{\ackname} This work was supported by the National Science Foundation under Grant No. 2143435 and by the National Science Foundation under Grant CCF 2236484.
\end{credits}

%% file: main.bbl
\begin{thebibliography}{10}
\providecommand{\url}[1]{\texttt{#1}}
\providecommand{\urlprefix}{URL }
\providecommand{\doi}[1]{https://doi.org/#1}

\bibitem{alahi2016social}
Alahi, A., Goel, K., Ramanathan, V., Robicquet, A., Fei-Fei, L., Savarese, S.:
  Social lstm: Human trajectory prediction in crowded spaces. In: Proceedings
  of the IEEE conference on computer vision and pattern recognition. pp.
  961--971 (2016)

\bibitem{angelopoulos2022cp}
Angelopoulos, A.N., Bates, S.: A gentle introduction to conformal prediction
  and distribution-free uncertainty quantification (2022)

\bibitem{bemporad2007robust}
Bemporad, A., Morari, M.: Robust model predictive control: A survey. In:
  Robustness in identification and control, pp. 207--226. Springer (2007)

\bibitem{chen2019crowd}
Chen, C., Liu, Y., Kreiss, S., Alahi, A.: Crowd-robot interaction: Crowd-aware
  robot navigation with attention-based deep reinforcement learning. In: 2019
  international conference on robotics and automation (ICRA). pp. 6015--6022.
  IEEE (2019)

\bibitem{chen2021interactive}
Chen, Y., Zhao, F., Lou, Y.: Interactive model predictive control for robot
  navigation in dense crowds. IEEE Transactions on Systems, Man, and
  Cybernetics: Systems  \textbf{52}(4),  2289--2301 (2021)

\bibitem{dietterich2022conformal}
Dietterich, T.G., Hostetler, J.: Conformal prediction intervals for markov
  decision process trajectories (2022)

\bibitem{dixit2023adaptive}
Dixit, A., Lindemann, L., Wei, S.X., Cleaveland, M., Pappas, G.J., Burdick,
  J.W.: Adaptive conformal prediction for motion planning among dynamic agents.
  In: Learning for Dynamics and Control Conference. pp. 300--314. PMLR (2023)

\bibitem{du2011robot}
Du~Toit, N.E., Burdick, J.W.: Robot motion planning in dynamic, uncertain
  environments. IEEE Transactions on Robotics  \textbf{28}(1),  101--115 (2011)

\bibitem{garcia1989model}
Garcia, C.E., Prett, D.M., Morari, M.: Model predictive control: Theory and
  practice—a survey. Automatica  \textbf{25}(3),  335--348 (1989)

\bibitem{hart1968formal}
Hart, P.E., Nilsson, N.J., Raphael, B.: A formal basis for the heuristic
  determination of minimum cost paths. IEEE Transactions on Systems Science and
  Cybernetics  \textbf{4}(2),  100--107 (1968)

\bibitem{helbing1995social}
Helbing, D., Molnar, P.: Social force model for pedestrian dynamics. Physical
  review E  \textbf{51}(5), ~4282 (1995)

\bibitem{huang2023neural}
Huang, Z., Chen, H., Pohovey, J., Driggs-Campbell, K.: Neural informed rrt*:
  Learning-based path planning with point cloud state representations under
  admissible ellipsoidal constraints. arXiv preprint arXiv:2309.14595  (2023)

\bibitem{huang2022learning}
Huang, Z., Li, R., Shin, K., Driggs-Campbell, K.: Learning sparse interaction
  graphs of partially detected pedestrians for trajectory prediction. IEEE
  Robotics and Automation Letters  \textbf{7}(2),  1198--1205 (2022).
  \doi{10.1109/LRA.2021.3138547}

\bibitem{incremona2017mpc}
Incremona, G.P., Ferrara, A., Magni, L.: Mpc for robot manipulators with
  integral sliding modes generation. IEEE/ASME Transactions on Mechatronics
  \textbf{22}(3),  1299--1307 (2017)

\bibitem{ji2022robust}
Ji, T., Geng, J., Driggs-Campbell, K.: Robust output feedback mpc with reduced
  conservatism under ellipsoidal uncertainty. In: 2022 IEEE 61st Conference on
  Decision and Control (CDC). pp. 1782--1789. IEEE (2022)

\bibitem{kamel2017robust}
Kamel, M., Alonso-Mora, J., Siegwart, R., Nieto, J.: Robust collision avoidance
  for multiple micro aerial vehicles using nonlinear model predictive control.
  In: 2017 IEEE/RSJ International Conference on Intelligent Robots and Systems
  (IROS). pp. 236--243. IEEE (2017)

\bibitem{karaman2011sampling}
Karaman, S., Frazzoli, E.: Sampling-based algorithms for optimal motion
  planning  \textbf{30}(7),  846--894 (2011)

\bibitem{KATH2021777}
Kath, C., Ziel, F.: Conformal prediction interval estimation and applications
  to day-ahead and intraday power markets. International Journal of Forecasting
   \textbf{37}(2),  777--799 (2021).
  \doi{https://doi.org/10.1016/j.ijforecast.2020.09.006},
  \url{https://www.sciencedirect.com/science/article/pii/S0169207020301473}

\bibitem{lin2024verification}
Lin, A., Bansal, S.: Verification of neural reachable tubes via scenario
  optimization and conformal prediction. Proceedings of Machine Learning
  Research vol vvv  \textbf{1}, ~16 (2024)

\bibitem{lindemann2023safe}
Lindemann, L., Cleaveland, M., Shim, G., Pappas, G.J.: Safe planning in dynamic
  environments using conformal prediction. IEEE Robotics and Automation Letters
   \textbf{8}(8),  5116--5123 (2023). \doi{10.1109/LRA.2023.3292071}

\bibitem{liu2023intention}
Liu, S., Chang, P., Huang, Z., Chakraborty, N., Hong, K., Liang, W., McPherson,
  D.L., Geng, J., Driggs-Campbell, K.: Intention aware robot crowd navigation
  with attention-based interaction graph. In: 2023 IEEE International
  Conference on Robotics and Automation (ICRA). pp. 12015--12021. IEEE (2023)

\bibitem{liu2021decentralized}
Liu, S., Chang, P., Liang, W., Chakraborty, N., Driggs-Campbell, K.:
  Decentralized structural-rnn for robot crowd navigation with deep
  reinforcement learning. In: 2021 IEEE international conference on robotics
  and automation (ICRA). pp. 3517--3524. IEEE (2021)

\bibitem{Muthali2023Multi-agent}
Muthali, A., Shen, H., Deglurkar, S., Lim, M.H., Roelofs, R., Faust, A.,
  Tomlin, C.: Multi-agent reachability calibration with conformal prediction.
  In: 2023 62nd IEEE Conference on Decision and Control (CDC). pp. 6596--6603
  (2023). \doi{10.1109/CDC49753.2023.10383723}

\bibitem{olsson2022estimating}
Olsson, H., Kartasalo, K., Mulliqi, N., Capuccini, M., Ruusuvuori, P.,
  Samaratunga, H., Delahunt, B., Lindskog, C., Janssen, E.A., Blilie, A.,
  et~al.: Estimating diagnostic uncertainty in artificial intelligence assisted
  pathology using conformal prediction. Nature communications  \textbf{13}(1),
  ~7761 (2022)

\bibitem{park2012robot}
Park, J.J., Johnson, C., Kuipers, B.: Robot navigation with model predictive
  equilibrium point control. In: 2012 IEEE/RSJ International Conference on
  Intelligent Robots and Systems. pp. 4945--4952. IEEE (2012)

\bibitem{qin2003survey}
Qin, S.J., Badgwell, T.A.: A survey of industrial model predictive control
  technology. Control engineering practice  \textbf{11}(7),  733--764 (2003)

\bibitem{sivakumar2021learned}
Sivakumar, A.N., Modi, S., Gasparino, M.V., Ellis, C., Velasquez, A.E.B.,
  Chowdhary, G., Gupta, S.: Learned visual navigation for under-canopy
  agricultural robots. arXiv preprint arXiv:2107.02792  (2021)

\bibitem{Sun2023ConformalPF}
Sun, J., Jiang, Y., Qiu, J., Nobel, P., Kochenderfer, M.J., Schwager, M.:
  Conformal prediction for uncertainty-aware planning with diffusion dynamics
  model. In: Neural Information Processing Systems (2023),
  \url{https://api.semanticscholar.org/CorpusID:268095750}

\bibitem{van2011reciprocal}
Van Den~Berg, J., Guy, S.J., Lin, M., Manocha, D.: Reciprocal n-body collision
  avoidance. In: Robotics Research: The 14th International Symposium ISRR. pp.
  3--19. Springer (2011)

\bibitem{vazquez2022conformal}
Vazquez, J., Facelli, J.C.: Conformal prediction in clinical medical sciences.
  Journal of Healthcare Informatics Research  \textbf{6}(3),  241--252 (2022)

\bibitem{vovk2005cp}
Vovk, V., Gammerman, A., Shafer, G.: Algorithmic learning in a random world.
  Springer Science \& Business Media (2005)

\bibitem{pmlr-v128-wisniewski20a}
Wisniewski, W., Lindsay, D., Lindsay, S.: Application of conformal prediction
  interval estimations to market makers’ net positions. In: Gammerman, A.,
  Vovk, V., Luo, Z., Smirnov, E., Cherubin, G. (eds.) Proceedings of the Ninth
  Symposium on Conformal and Probabilistic Prediction and Applications.
  Proceedings of Machine Learning Research, vol.~128, pp. 285--301. PMLR
  (09--11 Sep 2020)

\bibitem{zhang2024distributionfree}
Zhang, H., Ratliff, L.J., Dong, R.: Distribution-free guarantees for systems
  with decision-dependent noise (2024)

\end{thebibliography}
